\documentclass{article}
\usepackage{kr}

\usepackage{pdfpages}
\usepackage{times}
\usepackage{soul}
\usepackage{url}
\usepackage[hidelinks]{hyperref}
\usepackage[utf8]{inputenc}
\usepackage[small]{caption}
\usepackage{graphicx}
\usepackage{amsmath}
\usepackage{amsthm}
\usepackage{booktabs}
\usepackage{algorithm}
\usepackage{algorithmic}
\usepackage{booktabs}       %
\usepackage{amsfonts}       %
\usepackage{nicefrac}       %
\usepackage{microtype}      %
\usepackage{xspace}         %

\usepackage[table]{xcolor}
\usepackage{wrapfig}
\usepackage{tabularx}

\usepackage{etoolbox}    %

\newcommand{\go}{\textsc{GbO}\xspace}
\newcommand{\godel}{G\"odel\xspace}
\newcommand{\lukasiewicz}{Łukasiewicz\xspace}
\newcommand{\gt}{\textsc{GT}\xspace}
\newcommand{\bpg}{\textsc{GbOG}\xspace}
\newcommand{\ii}{implicit interpretation\xspace}

\newcommand{\godelint}{\mathcal{G}}
\newcommand{\boolint}{\mathcal{B}}
\newcommand{\propositions}{P}

\newcommand{\bgamma}{\boldsymbol{\gamma}}
\newcommand{\bmu}{\boldsymbol{\mu}}

\newtheoremstyle{compact}
  {1pt}   %
  {1pt}   %
  {\itshape}  %
  {}      %
  {\bfseries} %
  {.}     %
  { }     %
  {}      %

\theoremstyle{compact}
\newtheorem{theorem}{Theorem}
\theoremstyle{compact}
\newtheorem{corollary}{Corollary}[theorem]
\newtheorem{remark}{Remark}
\newtheorem{lemma}{Lemma}
\newtheorem{proposition}{Proposition}%
\newtheorem{definition}{Definition}

\AtBeginEnvironment{proof}{\vspace{-9pt}}
\AtEndEnvironment{proof}{\vspace{-7pt}}

\title{Gradient-Based Optimization on Gödel Logic as Discrete Local Search}

\author{
Alessandro Daniele$^{1,2}$\and
Emile van Krieken$^3$ \\
\affiliations
$^1$University of Bozen-Bolzano, Bozen, Italy\\
$^2$Fondazione Bruno Kessler, Trento, Italy\\
$^3$Vrije Universiteit Amsterdam, Amsterdam, Netherlands \\
\emails
alessandro.daniele@unibz.it,
e.van.krieken@vu.nl
}

\begin{document}

\maketitle

\begin{abstract}

A fundamental challenge in neurosymbolic systems is applying continuous gradient-based optimization to discrete logical domains. While fuzzy relaxations provide differentiability, they often lack a formal structural alignment with classical logic. In this work, we show that Gödel semantics addresses this limitation through a homomorphism that maps its continuous interpretations to Boolean ones, allowing discrete variables to be encoded while maintaining full differentiability. Building on this foundation, we show that gradient-based optimization on Gödel logic instantiates a discrete local search for Boolean satisfiability. Our formal analysis proves that each optimization step identifies and modifies a single variable within an unsatisfied clause, precisely mimicking the steps of a discrete solver. We identify local optima as the primary limitation of such dynamics and introduce the Gödel Trick, a stochastic reparameterization technique designed to improve the exploration of the solution space. We further show a formal connection between this approach, probabilistic inference, and the Gumbel-Max trick. Experimental results on SAT benchmarks and the Visual Sudoku task validate our theoretical findings, demonstrating that our approach effectively navigates complex combinatorial landscapes and provides a solid foundation for differentiable discrete search\footnote{Code available at: \url{https://github.com/DanieleAlessandro/Godel-Trick}}.
\end{abstract}

\begin{figure*}
    \centering
    \includegraphics[width=1.0\linewidth]{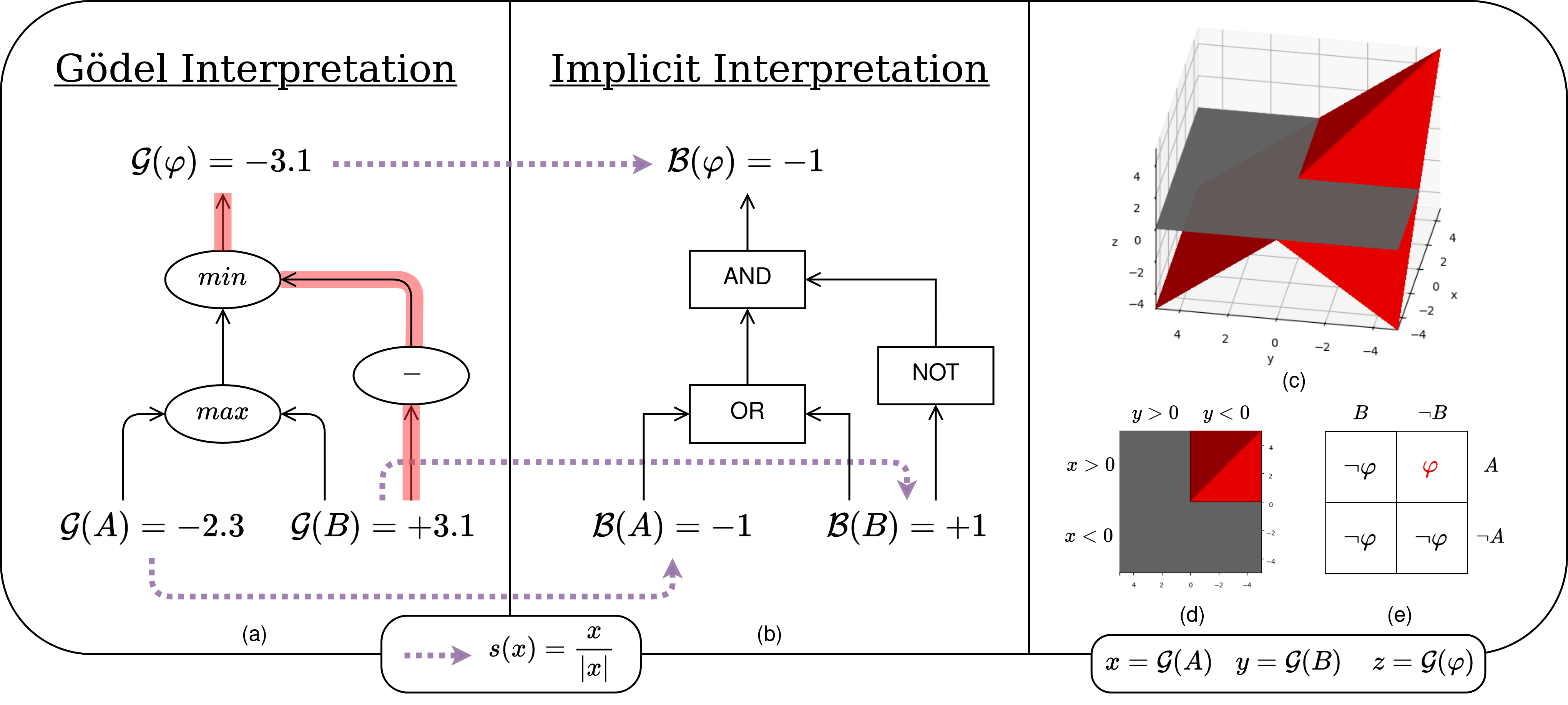}
    \caption{
    \textbf{The starting point of our work: there is a consistent mapping from \godel to boolean interpretations.}
    Example for formula $\varphi = (A \lor B) \land \lnot B$. (a) \godel interpretation (active path highlighted in red); (b) the corresponding boolean interpretation;
    (c-d) plot of \godel interpretation for $\varphi$ (red) and the plane $z=0$ (gray): (c) side view, (d) top view;
    (e) truth table of $\varphi$ in boolean logic. There is a one-to-one correspondence between the quadrants in (d) and the discrete values in (e).}
    \label{fig:virtual_imp}
\end{figure*}

\section{Introduction}

Deep learning has revolutionized artificial intelligence, driven by the power of Gradient-Based Optimization (\go) and its efficient implementation, that is the backpropagation~\cite{rumelhart1986learning}. While \go excels in continuous domains, its integration into neurosymbolic (NeSy) systems remains a significant challenge due to the inherently discrete and combinatorial nature of symbolic reasoning~\cite{feldstein2024mappingneurosymbolicailandscape,MARRA2024104062}. The difficulty does not lie merely in rendering logical constraints differentiable, but in doing so without sacrificing the structural rigor of classical logic.

A common strategy in the NeSy field is to employ fuzzy logics as continuous relaxations of Boolean operators~\cite{van2022analyzing,badreddine2022logic,daniele2019knowledge}. By allowing truth values to range across the unit interval $[0, 1]$, these ``soft'' logics provide a differentiable landscape suitable for \go. However, this transition frequently leads to a semantic mismatch: as no fuzzy logic can satisfy all properties of classical logic~\cite{gupta1991theory}, the resulting continuous approximations often diverge from the intended logical behavior, failing to capture the discrete essence of symbolic tasks~\cite{van2022analyzing}.

In this work, we demonstrate that \godel logic is a remarkable exception to this trend. Rather than acting as a mere ``soft'' surrogate, we prove that \godel semantics possesses unique algebraic properties that establish a structural bridge to Boolean logic. Specifically, we identify a homomorphism between \godel and Boolean semantics that allows for a consistent mapping between the continuous and discrete worlds. Our central argument is that \godel logic is ``discrete in disguise'': we prove that its truth function produces sparse gradients, such that each optimization step identifies and modifies only a single variable and, when applied to a Conjunctive Normal Form (CNF) formula, it identifies and modifies only a single unsatisfied clause. This reveals that Gradient-based Optimization on \godel logic (\bpg) does not simply approximate logic; it formally instantiates a discrete local search for Boolean satisfiability.

By characterizing \bpg as a discrete Local Search Algorithm (LSA) for SAT, we can identify its primary theoretical limitation: similar to deterministic solvers like GSAT, it is prone to converging to local optima. To overcome this, we introduce the \textit{\godel Trick} (GT), a stochastic reparameterization technique that introduces noise into the optimization process to improve exploration. We show that the \godel Trick is not merely a heuristic addition of noise, but provides a formal probabilistic grounding, acting as a Monte Carlo estimator for Weighted Model Counting (WMC). 

Our main theoretical contributions are threefold: (i) we prove the existence of a homomorphism between \godel and classical logic semantics; (ii) we provide a formal proof of the sparsity of the \godel gradient, demonstrating that gradient-based optimization on \godel logic circuits behaves as a LSA; (iii) we identify local optima as the fundamental barrier to convergence in this framework and introduce the \godel Trick (\gt), a stochastic variant of \godel optimization. Moreover, we establish its formal connection to probabilistic inference, showing it serves as an efficient Monte Carlo estimator for probabilistic interpretations and connecting it to the Gumbel-Max trick~~\cite{gumbel1954statistical,jang2022categorical}.

The remainder of the paper validates these theoretical findings through evaluation of \gt on SAT benchmarks~\cite{hoos2000satlib} and the Visual Sudoku task~\cite{augustine2022visual}.

\section{Background and notation}
\label{sec:background}

Propositional formulas are defined recursively as:
$
\varphi ::= p \mid \neg \varphi_1 \mid (\varphi_1 \land \varphi_2) \mid (\varphi_1 \lor \varphi_2),
$
where \( p \in \propositions \) is an atomic proposition, and \( \neg \), \( \land \), \( \lor \) denote negation, conjunction, and disjunction.

We consider two semantics: classical (Boolean) logic and \godel logic. In classical logic, formulas are assigned binary truth values in $\{-1, 1\}$~\footnote{Without loss of generality, we use $1$ and $-1$ as \textit{true} and \textit{false}, respectively}. \godel logic generalises classical logic by allowing truth values to be real numbers in the range $[0, 1]$, where $0$ represents false, $1$ represents true. In the remainder, we adopt truth values of formulas over the real number space $\mathbb{R}$ instead of 
$[0,1]$~\footnote{We consider logits in $\mathbb{R}$, that is, the outputs of a neural network before applying the sigmoid function. More details in the Supplementary Materials.}.

An interpretation is a function that maps formulas in truth values and is defined recursively as:

\begin{center}
    \renewcommand{\arraystretch}{1.2}
    \begin{tabular}{rcc}
    \hline
    \textbf{Formula} & \textbf{Boolean} $\boolint(\varphi)$ & \textbf{\godel} $\godelint(\varphi)$ \\
    \hline
    $p_i$ & $ \in \{-1, 1\}$ & $ \in \mathbb{R}$ \\
    $\lnot \varphi$ & $-\boolint(\varphi)$ & $-\godelint(\varphi)$ \\
    $\varphi_1 \land \varphi_2$ & $\min(\boolint(\varphi_1), \boolint(\varphi_2))$ & $\min(\godelint(\varphi_1), \godelint(\varphi_2))$ \\
    $\varphi_1 \lor \varphi_2$ & $\max(\boolint(\varphi_1), \boolint(\varphi_2))$ & $\max(\godelint(\varphi_1), \godelint(\varphi_2))$ \\
    \hline
    \end{tabular}
    \end{center}

Note that the interpretation of any formula is recursively determined by the interpretation of atomic propositions. In the following, we sometimes represent an interpretation compactly using a vector of truth values over the atomic propositions: \mbox{$\bmu = \langle \godelint(p_i) \rangle_{i=1}^n$},
and write $\godelint(\varphi; \bmu)$ to denote the value of formula \(\varphi\), assuming $\godelint(p_i) = \mu_i$. This notation is particularly useful for Theorem~\ref{th:prob_inf}, as it allows us to treat the interpretation as a fixed function on vector \(\bmu\). 
Similarly, we use \mbox{$\bgamma = \langle \boolint(p_i) \rangle_{i=1}^n$} and $\boolint(\varphi; \bgamma)$ for the boolean interpretation.

\section{Related Work}
\label{sec:related}

The neurosymbolic (NeSy) field focuses on incorporating logical knowledge into learning systems \cite{selsam2018learning,li2022nsnet,wang2019satnet}.

A prominent research direction investigates the continuous relaxation of logical constraints by adopting t-norm based fuzzy semantics. This framework enables the definition of differentiable logic layers, where logical formulas are translated into continuous functions suitable for gradient-based optimization. 
Logic Tensor Networks (LTN)~\cite{badreddine2022logic} and Semantic-Based Regularization (SBR)~\cite{diligenti2017semantic} inject prior knowledge into neural networks by embedding fuzzy logic in loss functions. 
These methods allow for the choice of different fuzzy semantics that have been extensively analyzed and compared in terms of expressiveness and optimisation~\cite{van2022analyzing,grespan2021evaluating,flinkow2024comparing,slusarz2023logic}. Furthermore, fuzzy logic operators can be used as part of the neural network architecture itself to incorporate background knowledge \cite{giunchiglia2024ccn+,daniele2019knowledge,daniele2023refining}.

\godel logic is frequently included as a standard semantic choice in various neurosymbolic frameworks~\cite{badreddine2022logic,daniele2023refining}, and as the only choice in other works~\cite{daniele2019knowledge,andreoni2025t}. However, in these contexts, it is primarily utilized as a continuous surrogate for Boolean constraints. Our contribution offers a novel perspective: by leveraging its specific algebraic properties, namely the existence of a homomorphism to Boolean logic, we show that gradient-based optimization on \godel semantics can be formally interpreted as a discrete local search. Our analysis of this structural alignment shifts the focus from treating \godel logic merely as a continuous relaxation to its specific optimization challenges, such as convergence to local optima. These issues are subsequently addressed by our proposed \godel Trick.

A separate but influential line of research deals with the probabilistic paradigm. 
NeSy methods that use probabilistic logics, such as Semantic Loss \cite{xu2018semantic}, DeepProbLog \cite{manhaeve2018deepproblog} and Semantic Probabilistic Layers \cite{ahmed2022semantic}, require computing the Weighted Model Counting (WMC) \cite{chavira2008probabilistic}, a \#P-hard problem \cite{valiant1979complexity}, at each iteration. 
To overcome this, methods either use compiled probabilistic circuits \cite{choi2020probabilistic,kisa2014probabilistic} or approximations, for example via neural networks \cite{van2023nesi} or by sampling \cite{NEURIPS2023_61202bb3}. While we do not perform a direct empirical comparison with probabilistic methods, as our primary focus is on providing a formal foundation and theoretical characterization of optimization within \godel logic, in Section~\ref{sec:prob_inference} we bridge the fuzzy and probabilistic paradigms by showing that the \godel Trick can be interpreted as a Monte Carlo estimator for WMC.

\section{Homomorphism from \godel to Boolean algebraic structures}
\label{sec:hom}

In this section, we demonstrate the existance of an homomorphism that maps continuous interpretations of \godel logic to discrete boolean interpretations of classical logic.
In the case of the \godel logic, we consider $\mathbb{R} \setminus \{0\}$, since this relation can only be proved when we exclude zero~\footnote{In practice, this is not an issue since we introduce noise during training, making the probability of $\godelint(\varphi)=0$ negligible (see Section~\ref{sec:godel-trick}).}.

Two algebraic structures can be defined: for clasical logic we have $\mathcal{L}_{\mathcal{B}} = \langle \{-1, 1\}, \min, \max, - \rangle$;  while for \godel semantics we define $\mathcal{L}_{\mathcal{G}} = \langle \mathbb{R} \setminus \{ 0 \}, \min, \max, - \rangle$. Both structures are De Morgan Lattice~\cite{moisil1935recherches} (also known as distributive i-lattices~\cite{kalman1958lattices}).

The sign function $s: \mathbb{R} \setminus \{0\} \to \{-1, 1\}$ is defined as:
$$
s(x) = \frac{x}{|x|} =
\begin{cases}
-1, & \text{if } x < 0, \\
\ \ \ 1, & \text{if } x > 0.
\end{cases}
$$

\begin{proposition}
\label{th:homomorphism}
The sign function $s(x)$ is a homomorphism from the \godel lattice $\mathcal{L}_{\mathcal{G}}$ to the boolean lattice $\mathcal{L}_{\mathcal{B}}$.
\end{proposition}

\begin{proof}
    We verify that $s$ preserves the lattice structure: 
    \paragraph{Negation:} Let $x \in \mathbb{R} \setminus \{ 0 \}$. Then, we have 
    $$s(-x) = -x/| -x | = -(x/|x|) = -s(x)$$ preserving the negation. 
    Note that this would not hold if we assumed zero to be a valid truth value. 
    \paragraph{Conjunction:} Let $x, y \in \mathbb{R} \setminus \{ 0 \}$. Then, 
    $$s(\min(x, y)) = \min(s(x), s(y))$$ preserving the conjunction. This holds because the minimum of two inputs will be negative iff at least one input is negative.
    \paragraph{Disjunction:}  Let $x, y \in \mathbb{R} \setminus \{ 0 \}$. Then,  $$s(\max(x, y)) = \max(s(x), s(y))$$ preserving the disjunction. This holds because the maximum is positive iff at least one value is positive. 
    
\end{proof}
As a consequence, the \godel interpretation of any formula can be mapped to a Boolean interpretation via $s$.

The relation defined by Proposition~\ref{th:homomorphism} plays a crucial role in enabling Gradient-Based Optimization (\go) to optimise logical formulas in a differentiable setting.
At each step of \go on $\godelint(\varphi)$, the current \godel interpretation can be mapped to a corresponding discrete interpretation via the sign function. We call this discrete interpretation the \textit{\ii}, defined as $\boolint(\varphi) = s(\godelint(\varphi))$.
Figure~\ref{fig:virtual_imp} shows an example: the implicit interpretation of formula $\varphi$ can be computed either by discretizing (via the sign function) each proposition and interpreting $\varphi$ with boolean semantics, or by interpreting it in \godel semantics and discretizing the final output.

\section{Gradient-based Optimization on \godel logic}
\label{sec:BP_LSA}

\begin{figure*}
    \includegraphics[width=\linewidth]{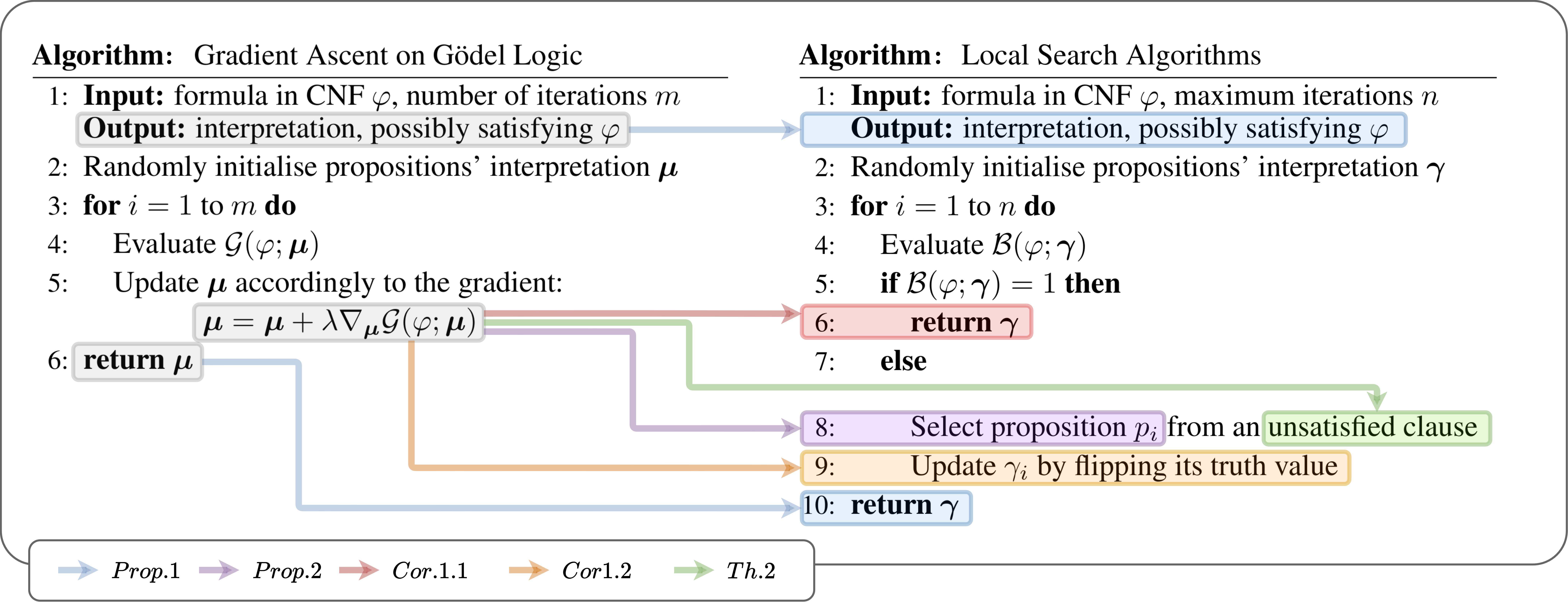}
    \caption{Gradient-based Optimization on \godel logic (\bpg) acts as a local search algorithm for SAT.}
    \label{fig:BP_vs_LSA}
\end{figure*}

In this section, we analyze the behaviour of Gradient-Based Optimization applied to \godel logic (\bpg), showing its correspondence with Local Search Algorithms (LSAs). Since the aim is to increase the value of the \godel interpretation, gradient ascent is applied on $\godelint(\varphi)$. As illustrated on the left side of Figure~\ref{fig:BP_vs_LSA}, the algorithm iteratively adjusts the continuous truth values according to the \emph{update rule} ${\boldsymbol{\mu} = \boldsymbol{\mu} + \lambda \nabla_{\boldsymbol{\mu}}\godelint(\varphi; \boldsymbol{\mu})}$, where $\lambda$ denotes the learning rate. Note that gradient ascent often updates the continuous interpretation without changing the sign of any proposition, leaving the implicit discrete interpretation unchanged. For this reason, we restrict our analysis to steps where a sign change occurs.

The alignment between \bpg and LSAs is illustrated in Figure~\ref{fig:BP_vs_LSA}, which maps each step of the two algorithms based on the following theoretical results.
Recall that every continuous interpretation in the \godel logic maps to a discrete one (Proposition~\ref{th:homomorphism}). 
We first prove that the gradients are sparse (Proposition~\ref{th:active_path}), which implies that each gradient ascent step modifies a single proposition. 
We then prove that if the formula is satisfied, the current interpretation is a fixed point of the update dynamics (Corollary~\ref{cor:1}). 
Otherwise, the modified variable moves toward the decision threshold via a gradient ascent step (Corollary~\ref{cor:2}), eventually changing its sign. 
Moreover, the modified variable always appears in an unsatisfied clause (Theorem~\ref{th:CNF}). These results collectively show that \bpg behaves as a discrete and deterministic LSA applied to the corresponding Boolean formula.

To prove the aforementioned alignment, we need to analyze the gradient of a \godel interpretation w.r.t. the propositions. With abuse of notation, we will refer to the interpretation of a formula by ignoring the function \( \godelint\) when computing the partial derivatives, e.g.,  $\frac{\partial \godelint(\varphi)}{\partial \godelint(\psi)}$ is reported as $\frac{\partial \varphi}{\partial \psi}$.

The first step is to show the sparsity of the gradients. To this end, we define the concept of path.

\begin{definition}[Path]
A \emph{path} \(\mathcal{P}\) from a formula \(\varphi\) to an atomic proposition \(p\) is a sequence of formulas:
${
\mathcal{P} = (\varphi, \psi_1, \psi_2, \dots, p)}
$,
where each \(\psi_i\) is a direct subformula of \(\psi_{i-1} \), and \(p\) is an atomic proposition.
\end{definition}

An example of a path $\mathcal{P}$ for formula $\varphi = (A \lor B) \land \lnot B$ is highlighted in red in Figure~\ref{fig:virtual_imp}(a): $$\mathcal{P} = \big((A \lor B) \land \lnot B, \lnot B, B\big)$$

\begin{lemma}
\label{th:n_even}
    The partial derivative \(\frac{\partial \varphi}{\partial \psi_i}\) along a path \(\mathcal{P} = (\varphi, \psi_1, \dots, p)\), takes values in \(\{-1, 0, 1\}\).
\end{lemma}
\begin{proof}
The partial derivative of $\varphi$ with respect to $p$ along the path is given by
$
   \frac{\partial \varphi}{\partial p} = \frac{\partial \varphi}{\partial \psi_1} \cdot 
   \cdot \dots \cdot \frac{\partial \psi_{n-1}}{\partial p}
$, 
obtained by applying the chain rule. Each node $\psi_i \in \mathcal{P}$ represents either a $\min$, $\max$, or negation, and the corresponding 
partial derivative $\frac{\partial \psi_i}{\partial \psi_{i+1}} \in \{-1, 0, 1\}$. Hence, their product is also in $\{-1, 0, 1\}$.
\end{proof}

\begin{remark}[Tie-breaking Rule]
    When multiple branches of a $\min$ or $\max$ operation result in the same truth value, we assume that a single branch is selected by the gradient.
 \end{remark}

\begin{definition}[Active path]
A path \(\mathcal{P} = (\varphi, \psi_1, \psi_2, \dots, p)\) is \emph{active} if the partial derivative $\frac{\partial \varphi}{\partial p}$ is different from zero.
\end{definition}

\begin{proposition}
\label{th:active_path}
Given a formula $\varphi$, there exists a unique active path \(\mathcal{P} = (\varphi, \psi_1, \dots, p)\) in the computational graph of \(\varphi\)
\footnote{This result previously appeared in a different form~\cite{van2022analyzing}.}.
\end{proposition}

\begin{proof}
    
    We proceed by induction on the structure of the formula.
    
    \paragraph{Base Case (Atomic Proposition):}
    If \(\varphi = p\), where \(p\) is an atomic proposition, we have $\frac{\partial \varphi}{\partial p} = 1$.
    There are no intermediate subformulas, so the path is \(\mathcal{P} = (p)\), and the gradient is trivially 1.
    
    \paragraph{Inductive Step (Negation):}
    Suppose the inductive hypothesis holds for a subformula \(\psi\); i.e., there exists a unique active path in the computational graph of $\psi$. Assume \(\varphi = \neg \psi\). The gradient of \(\varphi\) with respect to \(p\) is:
    \[
    \frac{\partial \varphi}{\partial p} = \frac{\partial \varphi}{\partial \psi} \cdot \frac{\partial \psi}{\partial p} = -\frac{\partial \psi}{\partial p}.
    \]
    which is different from zero, meaning the path is still active.

    \paragraph{Inductive Step (Conjunction/Disjunction):}
    Consider the case where \(\varphi = \psi_1 \land \psi_2\). By the inductive hypothesis, we know that there exists a unique active path from each subformula \(\psi_1\) and \(\psi_2\) to their respective atomic propositions \(p_1\) and \(p_2\), and that the properties hold along these paths.
    
    Since conjunction is interpreted as the minimum function, only the path corresponding to the subformula with the smallest truth value will have a non-zero gradient. Let’s assume \(\godelint(\psi_1) < \godelint(\psi_2) \), then the gradient with respect to \(p_1\) is:
    \[
    \frac{\partial \varphi}{\partial p_1} = \frac{\partial \varphi}{\partial \psi_1} \cdot \frac{\partial \psi_1}{\partial p_1} = \frac{\partial \psi_1}{\partial p_1}.
    \]
    and the gradient with respect to $p_2$ is:
    \[
    \frac{\partial \varphi}{\partial p_2} = \frac{\partial \varphi}{\partial \psi_2} \cdot \frac{\partial \psi_2}{\partial p_2} = 0.
    \]
    If  \(\godelint(\psi_1) > \godelint(\psi_2) \), then the gradient with respect to \(p_1\) is zero, while the gradient with respect to \(p_2\) is non-zero.
    The case for disjunction follows a similar reasoning, with the difference that the maximum function is used. 
    \end{proof}

As an example, consider Figure~\ref{fig:virtual_imp}(a): each partial derivative is zero, except for the path from $\godelint(\varphi)$ to $\godelint(B)$ (red path in the figure).

\begin{theorem}
\label{th:discrete_derivative}
Let $\varphi$ be a formula, and let $\pi$ be a subformula in the active path of $\varphi$. The partial derivative $\frac{\partial \varphi}{\partial \pi}$ is equal to the product of the \ii of the two formulas:
\[
\frac{\partial \varphi}{\partial \pi} = \boolint(\varphi) \cdot \boolint(\pi)
\]
\end{theorem}

\begin{proof}
    We proceed by induction on the structure of the formula \(\varphi\).
    
    \paragraph{Base Case:} If \(\varphi = \pi\), then: 
    \begin{align*}
        \frac{\partial \varphi}{\partial \pi} &= \frac{\partial \pi}{\partial \pi} = 1 = (\pm 1)^2 \\
        &= \boolint(\varphi)^2 = \boolint(\varphi) \cdot \boolint(\pi)
    \end{align*}

    \paragraph{Inductive Step (Negation):} Consider \(\varphi = \neg \psi\). By definition of the negation, $\godelint(\varphi) = - \godelint(\psi)$ and $\frac{\partial \varphi}{\partial \psi} = -1$. Applying the chain rule:
    \begin{align}
    \frac{\partial \varphi}{\partial \pi} &= \frac{\partial \varphi}{\partial \psi} \cdot \frac{\partial \psi}{\partial \pi} = - \frac{\partial \psi}{\partial \pi}\\
    &= - \boolint(\psi) \cdot \boolint(\pi) \label{eq:inductive} \\
    &= \ \ \ \boolint(\varphi) \cdot \boolint(\pi) \label{eq:negation}
    \end{align}
    where (\ref{eq:inductive}) corresponds to the inductive hypothesis, and (\ref{eq:negation}) is 
    a consequence of $\varphi = \neg \psi$. Hence, the proposition holds.

    \paragraph{Inductive Step (Conjunction/Disjunction):} Consider \(\varphi = \psi_1 \land \psi_2\). The derivative is non-zero only for the subformula with the smallest truth value, for which it is $1$. Assume \(\psi_1\) to have the smallest truth value, i.e.
    $$
    \godelint(\varphi) = \min(\godelint(\psi_1), \godelint(\psi_2)) = \godelint(\psi_1)
    $$
    then, we have:
    \begin{align}
    \frac{\partial \varphi}{\partial \pi} &= \frac{\partial \varphi}{\partial \psi_1} \cdot \frac{\partial \psi_1}{\partial \pi} = \frac{\partial \psi_1}{\partial \pi}\\
     &= \boolint(\psi_1) \cdot \boolint(\pi) \label{eq:inductive2} \\
     &= \boolint(\varphi) \cdot \boolint(\pi) \label{eq:min}
    \end{align}
    where (\ref{eq:inductive2}) holds by the inductive hypothesis, and (\ref{eq:min}) is a consequence of $\godelint(\varphi) = \godelint(\psi_1)$.
    Similarly, for disjunction, the derivative is non-zero for the subformula with the largest truth value, yielding the same result.
    \end{proof}

Theorem~\ref{th:discrete_derivative} shows that the direction of the gradient on a specific formula depends exclusively on the \ii, as the following two related corollaries highlight.

\begin{corollary}
\label{cor:1}
    Let $\varphi$ be a formula, and let $p$ be a proposition in the active path. If $\varphi$ is satisfied ($\boolint(\varphi) = 1$), then the implicit interpretation is not modified.
\end{corollary}
\begin{proof}
    By Theorem~\ref{th:discrete_derivative}, it holds: $\frac{\partial \varphi}{\partial p} = \boolint(p)$. Hence, the sign is not changed by the update rule: 
    \begin{align}
        \boolint'(p) &= s(\godelint(p) + \lambda \boolint(p)) \\
        &= s\big(\boolint(p) \cdot (|\godelint(p)| + \lambda)\big) = \boolint(p) \label{eq:sign_trick}
    \end{align}
    where $\lambda > 0$ is the learning rate, and $\boolint'(p)$ is the implicit interpretation after the update. In (\ref{eq:sign_trick}) we exploit the equality $\godelint(p) = s(\godelint(p)) \cdot |\godelint(p)| =\boolint(p) \cdot |\godelint(p)|$ to separate ${|\godelint(p)| + \lambda > 0}$ from the sign $\boolint(p)$.
\end{proof}

\begin{corollary}
\label{cor:2}
    Let $\varphi$ be a formula, and let $p$ be a proposition in the active path. If $\varphi$ is \emph{not} satisfied ($\boolint(\varphi) = -1$), then the gradient ascent step moves the variable toward the decision threshold, eventually flipping its sign.
\end{corollary}
\begin{proof}
    By Theorem~\ref{th:discrete_derivative}, it holds: $\frac{\partial \varphi}{\partial p} = - \boolint(p)$. Hence: 
    $$ \boolint'(p) = s(\godelint(p) - \lambda \boolint(p)) = s\big(\boolint(p) \cdot (|\godelint(p)| - \lambda)\big)$$
    If $|\godelint(p)| < \lambda$, then $\boolint'(p) = - \boolint(p)$. Otherwise, the sign is the same, but the magnitude is reduced:
     $${|\godelint'(p)| = |\godelint(p)| - \lambda < |\godelint(p)|}$$

\end{proof}

\begin{theorem}
\label{th:CNF}
Let $\varphi$ be a formula expressed in CNF:
$
\varphi = \bigwedge_{i=1}^{m} c_i
$,
where $c_i$ represent the i$^{th}$ clause, and $\boolint$ a boolean interpretation. If $\varphi$ is not satisfied ($\boolint(\varphi)=-1$), the active path passes through an unsatisfied clause.
\end{theorem}
\begin{proof}
First, note that $\varphi$ is a conjunction of clauses. As a consequence, $\godelint(\varphi) = \min_i \ \godelint(c_i)$, and  $\frac{\partial \varphi}{\partial c_k} \neq 0 \iff k = \arg\min_i \langle \godelint(c_i) \rangle$.
Let $c_k$ be a satisfied clause: $\godelint(c_k) > 0$. Since formula $\varphi$ is not satisfied, there must be at least an unsatisfied clause $c_j$ ($\godelint(c_j) < 0$). Then: $\godelint(c_j) < 0 < \godelint(c_k)$, and it must hold $\frac{\partial \varphi}{\partial c_k} = 0$.

\end{proof}

\section{The \godel Trick}
\label{sec:godel-trick}

In the previous section, we showed that \bpg mimics the behaviour of deterministic local search algorithms. 
When applied to a CNF, both methods iteratively select an unsatisfied clause, flipping the truth value of one of its literals. 
The difference is in
the selection heuristic: 
instead of relying on the boolean values, 
\bpg selects the variable based on continuous truth values given by the \godel interpretation.

However, much like LSAs converge to local minima, \bpg encounters analogous challenges. 
In the continuous \godel landscape, local maxima may emerge at the decision boundaries. 
This occurs when the gradients from different clauses point towards the threshold from opposite sides; since the threshold itself is excluded and never exactly reached, the variable is forced to continuously cross it. 
As the gradient flips its direction at each crossing, the system enters an oscillation. 
Consequently, what acts as a local optimum in the continuous space manifests as a \textit{cycle} between distinct discrete interpretations on opposite sides of the boundary.

This behavior is illustrated in Figure~\ref{fig:GT} for the formula $\varphi = (A \lor B) \land \lnot B$. 
When $\godelint(A) < 0 < \godelint(B)$, the gradient points in the direction of $\lnot B$. Conversely, when $\godelint(A) < \godelint(B) < 0$, the gradient points towards $B$. 
The result is a cycle between two unsatisfied discrete interpretations (${\lnot A \land \lnot B}$ and ${\lnot A \land B}$).

\begin{figure}
    \centering
    \includegraphics[width=0.8 \linewidth]{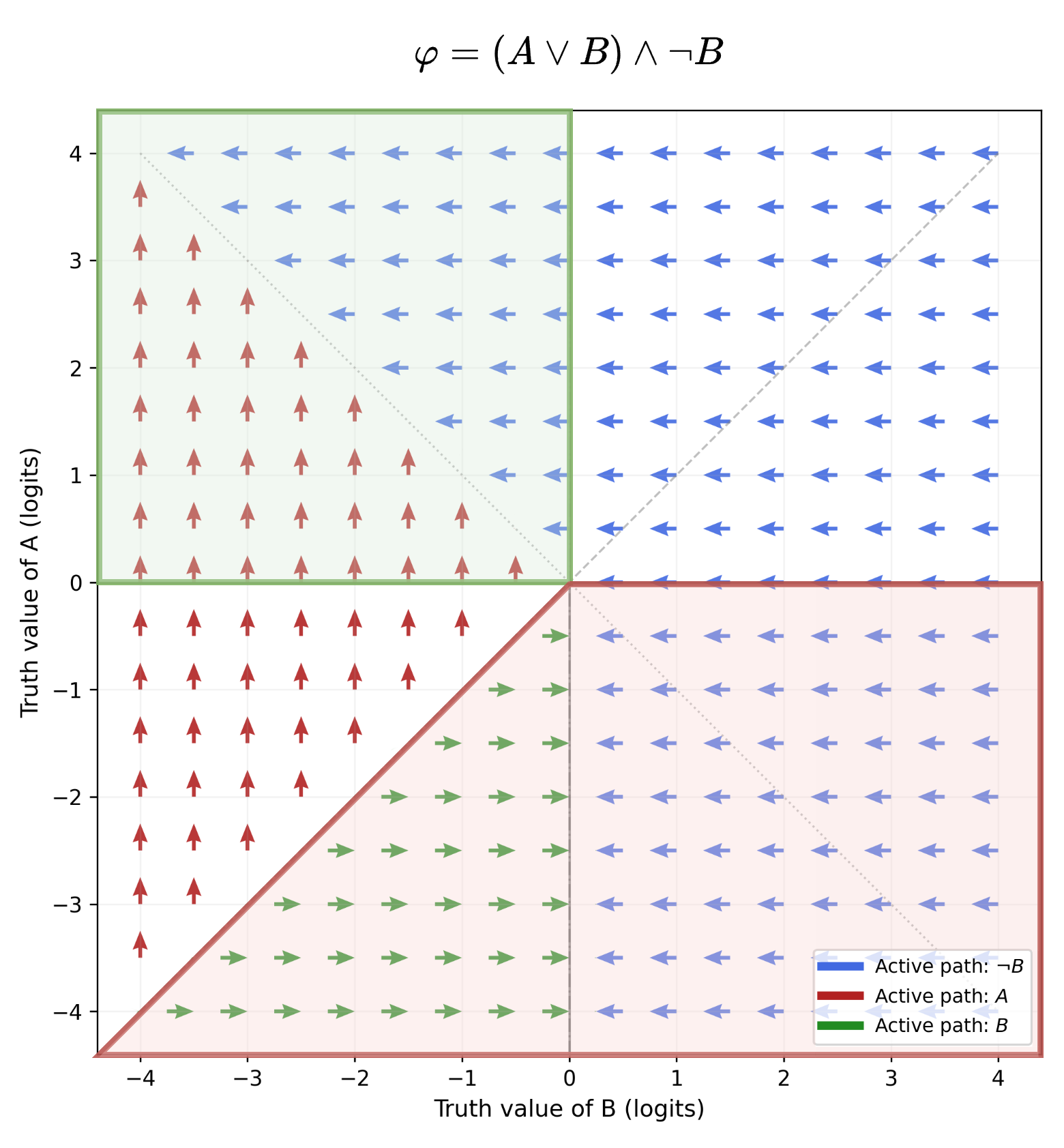}
    \vspace{-15pt}
    \caption{\textbf{Vector field of the gradient of a \godel interpretation} Formula $\varphi = (A \lor B) \land \lnot B$. Green zone: it corresponds to $\varphi$ being satisfied ($A \land \lnot B$); red zone: gradients points on opposite directions, forming a cycle between two unsatisfied interpretations: $\lnot A \land \lnot B$ and $\lnot A \land B$.}
    \label{fig:GT}
\end{figure}

\subsection{The \godel Trick}

The \textbf{\godel Trick} (\gt) introduces a controlled perturbation to the \godel interpretation of the propositions, enabling stochastic exploration of the solution space to escape local optima. Formally, we define the perturbed interpretation of a proposition $p$ as: $\godelint_\epsilon(p) = \ \godelint(p) + \epsilon$, where $\epsilon$ is a noise term preventing cycles. For a general formula, we use the recursive definition of \godel interpretation as defined in Section~\ref{sec:background}. 

The \godel Trick is a reparameterization trick \cite{DBLP:journals/corr/KingmaW13}, where the expected value of the gradient of a formula $\varphi$ corresponds to the gradient of the expected value of $\varphi$: $    \mathbb{E}_\epsilon \left[ \nabla_{\godelint(p)} \godelint_\epsilon(\varphi) \right] = \nabla_{\godelint(p)} \mathbb{E}_\epsilon \left[ \godelint_\epsilon(\varphi) \right]$. 
This can be derived from standard results on pathwise gradient estimators \cite{mohamed2020monte}.

\subsection{\godel Trick as approximate probabilistic inference}
\label{sec:prob_inference}

In Section~\ref{sec:hom}, we proved that the sign function $s$ is an homomorphism that maps a \godel interpretation to a discrete implicit one.
Similarly, a continuous probability distribution over \godel interpretation can be mapped to an \emph{implicit discrete distribution} over Boolean interpretations:
$\boolint_\epsilon(\varphi) = s(\godelint_\epsilon(\varphi))$.

For a proposition $p$, the probability $P\big( \mathcal{B}_\epsilon(p)\big)$ of $p$ being true under the implicit distribution $\boolint_\epsilon$ is:
\begin{align}
    P\big( \boolint_\epsilon(p) \big) 
    &= P\big( \godelint_\epsilon(p) > 0 \big) = P\big(\godelint(p) + \epsilon > 0\big)\\
    &= 1 - F_{\epsilon}(-\godelint(p)) = \theta_\epsilon(\godelint(p)). \label{eq:cdf}
\end{align}
where $F_{\epsilon}$ is the cumulative distribution function (CDF) of $\epsilon$, and $\theta_\epsilon(x) = 1 - F_\epsilon(-x)$ is the function that maps the unperturbed truth value $\godelint(p)$ of a proposition $p$, to the probability of $p$ being true according to $\boolint_\epsilon$.

We next establish a connection between the implicit distribution $\boolint_\epsilon$ and probabilistic inference. In this context, the probability of a formula being true is equivalent to its Weighted Model Count (WMC).

\begin{definition}[Probabilistic logic]
Let $\varphi$ be a formula defined over some propositions $p_i$. Let $\pi_i$ be the probability associated with proposition $p_i$. Then, the probability associated to $\varphi$ under probabilistic logic is defined as:
\begin{align}
    P(\varphi) &= \sum_{\boldsymbol{\omega} \in \{-1, 1 \}^n} H\big( \boolint(\varphi; \boldsymbol{\omega})\big) P(\boldsymbol{\omega})
\label{eq:expected_value_bool} \\
    P(\boldsymbol{\omega}) &= \prod_{i = 1}^n \pi_i^{H(\omega_i)} (1 - \pi_i)^{1 - H(\omega_i)}\label{eq:prod_bool}
\end{align}
where $H$ is the Heaviside function that maps positive values to one, and negative values to zero: $H(x) = \mathbf{1}(x>0)$.
\end{definition}

The connection between \gt and probabilistic logic is given by the following theorem.

\begin{theorem}
\label{th:prob_inf}

Let $\varphi$ be a formula defined over some propositions $p_i$. Let $\varphi$ be interpreted under probabilistic logic, with $\pi_i$ the probability associated with proposition $p_i$, and let $\epsilon$ be a continuous distribution over the real numbers. 

If we define the proposition's truth values as $\mathcal{G}(p_i) = \theta_\epsilon^{-1}(\pi_i)$, then:
$
P\big( \mathcal{B}_\epsilon(\varphi)\big) = P(\varphi)
$,
where $P(\varphi)$ is the probability of $\varphi$ being true under the probabilistic interpretation. 
\end{theorem}

\begin{proof}
We first consider the probability of $\boolint_\epsilon(\varphi)$ being positive. Such a probability can be defined as the expected value of the formula's truth value, assuming it to be in $\{0, 1\}$ rather than $\{-1, 1\}$. Let $\boldsymbol{\epsilon}$ be the vector of sampled noise, and $\boldsymbol{\mu}$ the corresponding vector of noisy interpretations:
$$\mu_i = \godelint_\epsilon(p_i) = \godelint(p_i) + \epsilon_i$$ Then:
\begin{align}
    P(\boolint_\epsilon(\varphi)) &= \mathbb{E}_{\boldsymbol{\epsilon}} \big[ H\big( \boolint_\epsilon(\varphi) \big) \big]\\
    &= \mathbb{E}_{\boldsymbol{\epsilon}} \big[ H\big( s( \godelint(\varphi; \boldsymbol{\mu}) ) \big) \big]
\end{align}
where $H$ is the Heaviside function, $\godelint(\varphi; \boldsymbol{\mu})$ is the \godel interpretation of $\varphi$ (as defined in~ Section~\ref{sec:background}) assuming $\godelint(p_i) = \mu_i$.

\noindent By applying Proposition~\ref{th:homomorphism}:
\begin{align}
    P(\boolint_\epsilon(\varphi))
    &= \mathbb{E}_{\boldsymbol{\epsilon}} \big[ H\big( s( \godelint(\varphi; \boldsymbol{\mu}) ) \big) \big] \\
    &= \mathbb{E}_{\boldsymbol{\epsilon}} \big[ H\big( \boolint(\varphi; s(\boldsymbol{\mu}))  \big) \big]\\
    &= \mathbb{E}_{\boldsymbol{\gamma}} \big[ H\big( \boolint(\varphi; \boldsymbol{\gamma})  \big) \big] \label{eq:change_of_variable}
\end{align}
where in (\ref{eq:change_of_variable}) we applied a change of variable on the expectation by defining $\gamma_i = s(\mu_i)$. Note that $\boldsymbol{\gamma} \in \{ -1, 1 \}^n$, and the expected value can be represented as a summation:
\begin{equation}
\label{eq:exected_value_godel}
    P(\boolint_\epsilon(\varphi)) = \sum_{\boldsymbol{\gamma} \in \{-1, 1 \}^n} H\big( \boolint(\varphi; \boldsymbol{\gamma})\big) P(\boldsymbol{\gamma})
\end{equation}
with
\begin{equation}
\label{eq:prod_godel}
    P(\boldsymbol{\gamma}) = \prod_{i = 1}^n P(\gamma_i)^{H(\gamma_i)} (1 - P(\gamma_i))^{1 - H(\gamma_i)}
\end{equation}

Equations (\ref{eq:prod_godel}) and (\ref{eq:prod_bool}) are identical, except for the usage of $P(\gamma_i)$ instead of $\pi_i$. We need to prove the equivalence between these two probabilities:
\begin{align}
    P(\gamma_i) &= P(s(\mu_i)) \\
    &= P(\godelint(p_i) + \epsilon_i > 0) \label{eq:mu_def} \\
    &= \theta_\epsilon(\godelint(p_i)) \label{eq:cdf_usage} \\
    &=\theta_\epsilon(\theta_\epsilon^{-1}(\pi_i)) = \pi_i \label{eq:G_e_definition}
\end{align}
where (\ref{eq:mu_def}) is obtained by applying the definition of $\mu_i$, (\ref{eq:cdf_usage}) derives from (\ref{eq:cdf}), and (\ref{eq:G_e_definition}) comes from the definition of $\godelint(p_i)$ in theorem statement.

As a consequence:
\begin{align}
P(\boolint_\epsilon(\varphi)) &= \sum_{\boldsymbol{\gamma} \in \{-1, 1 \}^n} H\big( \boolint(\varphi; \boldsymbol{\gamma})\big) P(\boldsymbol{\gamma}) \\
&= \sum_{\boldsymbol{\omega} \in \{-1, 1 \}^n} H\big( \boolint(\varphi; \boldsymbol{\omega})\big) P(\boldsymbol{\omega}) = P(\varphi)
\end{align}
\end{proof}

Theorem~\ref{th:prob_inf} formally connects GT with probabilistic inference by showing that it acts as a Monte Carlo estimator for WMC. This result provides a theoretical justification for our approach and establishes a basis for future research directions (see Section~\ref{sec:conclusion}). However, Theorem~\ref{th:prob_inf} does not imply that GT is also an estimator of the WMC gradient; the specific nature of its optimization dynamics, which remains rooted in local search, is further clarified in the following remark.

\begin{remark}
    Unlike \bpg, \gt optimizes a probability distribution over assignments rather than a single state. Therefore, it can not be properly framed as a LSA. Nonetheless, the local search intuition persists: when a sampled solution violates a clause, the gradient penalizes the responsible literal's truth value within the distribution parameters. Effectively, this performs literal flips in expectation rather than deterministically.
\end{remark}

\subsection{The choice of the noise distribution}
The noise $\epsilon$ plays a crucial role in the \gt. We consider two options: logistic and uniform distributions.

\paragraph{Logistic distribution.}
\label{paragraph:logistic_distribution}
For a standard logistic distribution, the corresponding CDF is the sigmoid function $F_\epsilon(x) = \sigma(x) = \frac{1}{1+e^{-x}}$, hence:

\begin{align}
    \theta_\epsilon(\godelint(p)) &= 1 - F_\epsilon(-\godelint(p))\\
    &= 1 - \sigma(-\godelint(p)) = \sigma(\godelint(p))
\end{align}

Thus, in the logistic noise case, $\godelint(p)$ can be interpreted as the logit of the probability of $p$ being true, establishing a direct link between the \gt and probabilistic inference.
Moreover, since the logistic distribution can be expressed as the difference of two Gumbel distributions, one can establish a connection between the \godel Trick and the Gumbel-Max Trick~\cite{gumbel1954statistical}, a common reparameterization method for estimating expected values in discrete settings (more details in Section~\ref{appendix:gumbel}).

\paragraph{Uniform distribution.}
The probability mass of the logistic distribution is concentrated close to the center, potentially reducing the exploration effect of the noise.
Therefore, we also consider the uniform distribution, which is more spread out, potentially increasing the exploration of the solution space.

Note that the probabilistic interpretation of \gt holds for other distributions than the logistic, as highlighted by Equation~\ref{eq:cdf}. However, we can no longer directly interpret $\godelint(p)$ as the logits of $P\big( p \big)$ as it induces a different discrete distribution. Specifically, if $\epsilon \sim \mathcal{U}(a, b)$:
\begin{equation}
    \theta_\epsilon(x) = \begin{cases}
        0 \quad \quad \ \text{if} \ x < -b \\
        \frac{x + b}{b - a} \quad \text{if} \ x \in [-b, -a]\\
        1 \quad \quad \ \text{if} \ x > -a.
    \end{cases}
\end{equation}
Given a probability $\pi$ of a proposition $p$, the unperturbed truth value is $\theta_\epsilon^{-1}(\pi) = \pi \cdot (b - a) - b$.

\section{\godel Trick with Categorical variables}
\label{sec:categorical}
The \godel Trick proposed in the Section~\ref{sec:godel-trick} is defined exclusively on boolean variables. Despite its generality, in the context of NeSy, categorical variables are frequently used (see as an example Section~\ref{sec:sudoku}), for which the \godel Trick is not directly applicable. Let $\pi_i$, with $i \in [1,\dots K]$, be the probabilities associated with $K$ categories, such that $\sum_{i=1}^K \pi_i = 1$. We can define a propositional variable $p_i$ for each category, and use the \godel Trick: $\godelint(p_i) = \theta_\epsilon^{-1}(\pi_i)$. However, by doing so we are implicitly assuming independence between the propositions. On the contrary, we need to assure that exactly one category is true. To include such a constraint, a possibility is to enforce the satisfaction of the XOR between the $K$ propositions trough logical constraints. However, such a solution could require to include a large amount of rules, reducing the efficiency of the method.

We propose a different strategy (see Section~\ref{sec:sudoku} for an empirical evaluation): we define a $\mathrm{shift}$ function that shifts the perturbed truth values $\godelint_\epsilon(p_i)$ of propositions $p_i$ in order to enforce the aforementioned constraint:
$$
\mathrm{shift}({\bf x}) = {\bf x} - \frac{x_i + x_j}{2},
$$
where $i$ and $j$ are the indexes of the highest and second highest elements of the vector ${\bf x}$, respectively. Note that for any vector ${\bf x}$, the shifted vector $\bar{\bf x} = \mathrm{shift}({\bf x})$ contains exactly one value larger than zero:
\begin{align}
\bar{x}_i &= x_i -  \frac{x_i + x_j}{2} = \frac{x_i - x_j}{2} > 0 \\
\bar{x}_j &= x_j -  \frac{x_i + x_j}{2} = \frac{x_j - x_i}{2} < 0 \\
\bar{x}_k &\leq \bar{x_j} < 0 \quad \forall k \neq i
\end{align}

It is worth noting two properties of the shifted vector: first, the highest element of $\bar{\bf x}$ is the only one higher than zero; second, the second-highest element is the negation of the first: $\bar{x}_j = - \bar{x}_i$. 

We define the vector $\godelint_\epsilon({\bf p})$ of perturbed truth values as:
$$\godelint_\epsilon({\bf p}) = \left\langle \godelint_\epsilon(p_1),\godelint_\epsilon(p_2) \dots \godelint_\epsilon(p_K) \right\rangle$$

We also define its shifted version as: $\bar{\godelint}_\epsilon({\bf p}) = \mathrm{shift}(\godelint_\epsilon({\bf p}))$. Because of the properties of the $\mathrm{shift}$ function, we have:
$\bar{\godelint}_\epsilon(p_i) > 0$ for the highest value $i$ in $\godelint_\epsilon({\bf p})$, and $\forall k \neq i \ \ \ \bar{\godelint}_\epsilon(p_k) < 0$. As a consequence, the constraint of the categorical variable is satisfied (i.e., exactly one value is true). Since the shifted truth values are still used in combination with \godel Logic:
\begin{equation}
    \bar{\godelint}_\epsilon(p_i) = - \bar{\godelint}_\epsilon(p_j) = \bar{\godelint}_\epsilon(\lnot p_i)
\end{equation}
representing the categorical variable as a Bernoulli variable that select one among the two highest values.

Finally, when using the Gumbel distribution to generate noise, the resulting behavior is equivalent to the one of the Gumbel-Max Trick (see Section~\ref{appendix:gumbel}).

\section{Connections with Gumbel-Max Trick}
\label{appendix:gumbel}
The Gumbel-Max Trick is a well-known reparameterization approach for sampling from a categorical distribution~\cite{gumbel1954statistical,jang2022categorical}. Given a categorical distribution defined by probabilities $\pi_i$, the trick allows for an efficient way to sample from this distribution by adding continuous noise to its logits.

Starting from the softmax logits ${\bf z} = \left\langle z_1, z_2 \dots z_K\right\rangle$ of the probabilities $\pi_i$, one can add noise sampled from a standard Gumbel distribution to the logits, and the distribution of the argmax of the resulting vector corresponds to the original categorical distribution. Consider the following equation:
\begin{equation}
    {\bf x} = \mathrm{onehot}(\arg\max_i (z_i + \epsilon_i))
\end{equation}
with $\epsilon_i \sim \text{Gumbel}(0, 1)$. The key aspect of the Gumbel-Max Trick is that the probability that $\arg\max_i (z_i + \epsilon_i) = j$ is equivalent to $\pi_j$.

Now, consider replacing the $\arg\max$ function with the $\mathrm{shift}$ function, and applying the sign function $s$ to the resulting vector:
\begin{equation}
    {\bf v} = s\big(\mathrm{shift}\big(\left\langle z_i + \epsilon_i \right\rangle \big)\big)
\end{equation}
with the noise still being sampled from a standard Gumbel distribution. As proven in section~\ref{sec:categorical}, the $\mathrm{shift}$ function translates the values of the vector in such a way that only the highest is greater than zero. As a consequence, the sign function $s$ maps all the values to $-1$, except for the $\arg\max$, which is mapped to $+1$.

Note that, in our framework, we defined $\top$ to be $+1$, and $\bot$ to be $-1$, different from the context of the Gumbel-Max Trick, where $\bot$ is defined as $0$. Therefore, vector ${\bf v}$ can be seen as a rescaled version of ${\bf x}$ suitable for the application of the \godel Trick:
\begin{equation}
    {\bf v} = 2 {\bf x} - 1
\end{equation}

\label{sec:exp}

\section{Experiments on SAT}
\label{sec:SAT}

In this section, we evaluate the performance of the \gt method on SAT problems to validate our theoretical findings. SAT problems contain strong dependencies among clauses, which often cause fuzzy logics to get stuck in local optima. This makes them a natural stress test for \gt, which is specifically designed to escape such local optima.

We focus on satisfiable instances from the SATLIB library~\cite{hoos2000satlib}, since \gt is an incomplete solver. Note that most problems in SATLIB are satisfiable, and many SAT solvers are themselves incomplete.

\paragraph{Setup.}
SATLIB is a well-established repository of SAT benchmarks, organized into collections representing different domains. For example, the UF family contains random 3-SAT problems, while the Planning set includes SAT-encoded planning instances.

We compare \gt against three fuzzy semantics: Product, \lukasiewicz, and \godel logics. These are natural baselines, as \gt can be applied in the same contexts and it is proposed as an alternative in the NeSy context.

\paragraph{Methodology.}
\begin{figure}   %
    \centering
    \includegraphics[width=0.8\linewidth]{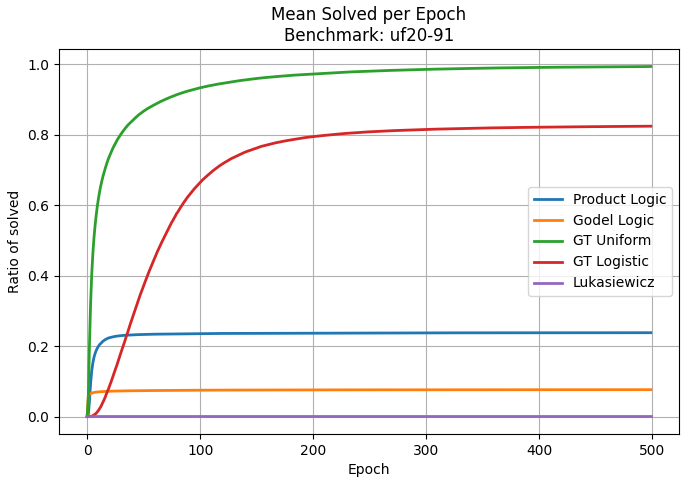}
    \caption{Ratio of solved problems on the uf20-91 benchmark of SATLIB. Average of 100 runs over 50k epochs.}
    \label{fig:plot_results}
\end{figure}

Initially, we conducted a grid search to optimise learning rate and momentum. This search was performed separately for each method on the simplest benchmark of the UF collection: \texttt{uf20-91}. For the noise, we used the standard logistic and uniform distributions, since our preliminary experiments showed that GT is not particularly sensitive to the choice of variance. We run these starting experiments on an NVIDIA RT2080Ti with 11GB of RAM.

Figure \ref{fig:plot_results} shows the percentage of solved problems for each method as a function of the number of training epochs. The reported averages are computed over 100 samples (equivalent to 100 independent runs), each evaluated on the 1000 instances of SAT in \texttt{uf20-91}. Among the baselines, Product performs best, followed by \godel, while \lukasiewicz fails to solve any instance. 

We then applied all methods, excluding \lukasiewicz due to its complete failure, to the remaining SATLIB benchmarks. The results are summarized in Table \ref{tab:methods_comparison}. The table reports the average performance across benchmark collections.

\begin{table*}[t]
    \centering
       \caption{Comparative table showing results on SATLIB benchmarks for four methods: Product Logic, \godel Logic, \godel Trick Logistic, and \godel Trick Uniform. Columns: S stands for Sample Solved (mean $\pm$ std of number of solved samples), B stands for Best Solution f(mean $\pm$ std of number of problem solved by keeping the best results among the samples). Highest S and B in bold.}
    \label{tab:methods_comparison}
    \scshape
    \resizebox{\textwidth}{!}{%
\begin{tabular}{l|rr|rr|rr|rr}
\toprule
\textbf{Domain} & \multicolumn{2}{c|}{\textbf{Product Logic}} & \multicolumn{2}{c|}{\textbf{\godel Logic}} & \multicolumn{2}{c|}{\textbf{GT Logistc}} & \multicolumn{2}{c}{\textbf{GT Uniform}} \\
 & \textbf{S(\%)} & \textbf{B(\%)} & \textbf{S(\%)} & \textbf{B(\%)} & \textbf{S(\%)} & \textbf{B(\%)} & \textbf{S(\%)} & \textbf{B(\%)} \\
\hline
UF & $2.9 \pm 7.4$ & $6.6 \pm 15.0$ & $0.9 \pm 2.4$ & $11.5 \pm 29.6$ & $25.0 \pm 25.8$ & $57.5 \pm 24.3$ & $\mathbf{74.5} \pm \mathbf{12.8}$ & $\mathbf{99.4} \pm \mathbf{\phantom{0}0.7}$ \\
\rowcolor{gray!20}
RTI/BMS & $0.0 \pm 0.0$ & $0.0 \pm \phantom{0}0.0$ & $0.0 \pm 0.0$ & $0.0 \pm 0.0$ & $8.3 \pm 4.2$ & $28.3 \pm 12.3$ & $\mathbf{46.2} \pm \mathbf{23.4}$ & $\mathbf{95.6} \pm \mathbf{\phantom{0}4.4}$ \\
CBS & $0.0 \pm 0.0$ & $0.0 \pm \phantom{0}0.0$ & $0.0 \pm 0.0$ & $0.0 \pm 0.0$ & $6.7 \pm 10.6$ & $24.8 \pm 25.1$ & $\mathbf{70.3} \pm \mathbf{19.3}$ & $\mathbf{100.0} \pm \mathbf{\phantom{0}0.1}$ \\
\rowcolor{gray!20}
FLAT & $0.0 \pm 0.0$ & $0.0 \pm \phantom{0}0.0$ & $0.0 \pm 0.0$ & $1.2 \pm 2.4$ & $5.4 \pm 10.7$ & $20.3 \pm 39.9$ & $\mathbf{41.8} \pm \mathbf{34.8}$ & $\mathbf{80.2} \pm \mathbf{39.6}$ \\
SW & $0.0 \pm 0.0$ & $0.0 \pm \phantom{0}0.0$ & $0.0 \pm 0.0$ & $0.0 \pm 0.0$ & $0.0 \pm 0.0$ & $0.0 \pm 0.0$ & $\mathbf{30.2} \pm \mathbf{44.7}$ & $\mathbf{38.0} \pm \mathbf{46.9}$ \\
\rowcolor{gray!20}
PLANNING & $0.0 \pm 0.0$ & $0.0 \pm \phantom{0}0.0$ & $0.0 \pm 0.0$ & $0.0 \pm 0.0$ & $6.8 \pm 6.8$ & $21.4 \pm 21.4$ & $\mathbf{9.9} \pm \mathbf{9.9}$ & $\mathbf{28.6} \pm \mathbf{28.6}$ \\
AIS & $0.0 \pm 0.0$ & $0.0 \pm \phantom{0}0.0$ & $0.0 \pm 0.0$ & $0.0 \pm 0.0$ & $0.8 \pm 0.0$ & $25.0 \pm 0.0$ & $\mathbf{29.8} \pm \mathbf{0.0}$ & $\mathbf{75.0} \pm \mathbf{\phantom{0}0.0}$ \\
\rowcolor{gray!20}
QG & $0.0 \pm 0.0$ & $0.0 \pm \phantom{0}0.0$ & $0.0 \pm 0.0$ & $0.0 \pm 0.0$ & $\mathbf{0.4} \pm \mathbf{0.0}$ & $\mathbf{10.0} \pm \mathbf{0.0}$ & $0.0 \pm 0.0$ & $0.0 \pm \phantom{0}0.0$ \\
BMC & $0.0 \pm 0.0$ & $0.0 \pm \phantom{0}0.0$ & $0.0 \pm 0.0$ & $0.0 \pm 0.0$ & $0.0 \pm 0.0$ & $0.0 \pm 0.0$ & $0.0 \pm 0.0$ & $0.0 \pm \phantom{0}0.0$ \\
\rowcolor{gray!20}
DIMACS & $0.0 \pm 0.0$ & $0.0 \pm \phantom{0}0.0$ & $0.0 \pm 0.0$ & $0.0 \pm 0.0$ & $0.0 \pm 0.0$ & $0.0 \pm 0.0$ & $\mathbf{1.4} \pm \mathbf{2.3}$ & $\mathbf{27.4} \pm \mathbf{35.9}$ \\
BEIJING & $0.0 \pm 0.0$ & $0.0 \pm \phantom{0}0.0$ & $0.0 \pm 0.0$ & $0.0 \pm 0.0$ & $5.9 \pm 0.0$ & $\mathbf{15.4} \pm \mathbf{0.0}$ & $\mathbf{14.8} \pm \mathbf{0.0}$ & $\mathbf{15.4} \pm \mathbf{\phantom{0}0.0}$ \\
\bottomrule
\end{tabular}
    }
\end{table*}

 Given that each sample can be run in parallel, we explore multiple initial interpretations efficiently. For each problem, we evaluated 100 samples (equivalent to 100 independent runs) and computed two metrics: \textit{S}, the percentage of successfully solved instances across all samples, and \textit{B}, the number of solved problems considering the best result among the 100 samples for each problem. Note that each sample is computed in parallel on the GPU, and memory usage remains minimal. Consequently, the \textit{B} measure can be further improved by increasing the number of samples.

\paragraph{Results.}
The results clearly demonstrate that \gt yields a substantial improvement over all baselines, with the Uniform variant achieving the best performance. These results align with our theoretical analysis: while \godel logic provides the correct gradient direction for a discrete flip, it lacks a mechanism to escape local optima in complex combinatorial landscapes. The noise introduced in \gt-based models facilitates this exploration, allowing the solver to ``jump'' between different regions of the Boolean hypercube.

The results also show an advantage of \gt over Product logic. The latter is often viewed as an approximation of probabilistic logic that assumes independence between clauses. In contrast, while \gt also possesses a probabilistic interpretation (see Theorem~\ref{th:prob_inf}), it does not make such an assumption. This property explains the difference in performances on SAT problems, which are characterized by tight variable dependencies. 

Despite these gains, certain benchmarks like {\scshape bmc} and {\scshape qg} remain challenging for \gt. These results suggest that, while \gt provides a powerful differentiable alternative to traditional fuzzy logics, very large-scale combinatorial problems may still require the integration of more advanced search heuristics, which we leave for future investigation.

\section{Visual Sudoku}
\label{sec:sudoku}

In addition to SAT problems, we evaluated our framework on the Visual Sudoku task \cite{augustine2022visual}, a NeSy benchmark where the goal is to classify the validity of a $9\times9$ grid represented by MNIST images, where the only available supervision is the global validity of the board. 

The architecture stacks a GT layer on top of a neural perception network. Specifically, a CNN maps each MNIST image in the grid to propositions $p_{i,k}$, where $i$ and $k$ denote the cell index and the digit, respectively. These outputs are then processed by the \gt-based reasoner to evaluate the satisfaction of the Sudoku rules. The logical constraints enforce that no two cells $i$ and $j$ in the same row, column, or 3×3 block can contain the same digit $k$, which we encode via the formula $\phi_{i,j,k} = \lnot p_{i,k} \lor \lnot p_{j,k}$.

We compare our approach with a CNN, NeuPSL \cite{pryor2022neupsl}, and A-NeSI \cite{van2023nesi}. \footnote{We use the baseline results reported by \cite{van2023nesi} for the CNN and NeuPSL. We run A-NeSI on our machine to compare execution time.} The experiments were conducted on a machine equipped with an NVIDIA GTX 3070 with 12GB RAM.
For the perception neural network, we use the same architecture as A-NeSI. For the \godel logic, we implement the interpretation in the log-space to ensure numerical stability. For \gt, we apply the categorical $shift$ function described in Section \ref{sec:categorical}. In both cases, training is performed in two phases: we first optimize clauses independently to provide a denser gradient signal, then aggregate them into the full Sudoku formula. A comprehensive description of the architecture and evaluation settings is provided in the Supplementary Material.

\begin{table}[h]
\centering
\caption{Comparisons on the $9\times9$ Visual Sudoku task.}
\label{tab:sudoku_results}
\begin{tabular}{lcc}
\toprule
\textbf{Method} & \textbf{Accuracy (\%)} & \textbf{Avg. Time (min)} \\
\midrule
CNN (Perception) & $51.20 \pm 2.20$ & - \\
NeuPSL & $51.50 \pm 1.37$ & - \\
A-NeSI & $62.25 \pm 2.20$ & $20.3$ \\
\midrule
Gödel Logic (Det.) & $61.19 \pm 1.61$ & $20.5$ \\
\textbf{Gödel Trick (GT)} & $\mathbf{62.95 \pm 1.42}$ & $\mathbf{8.5}$ \\
\bottomrule
\end{tabular}
\end{table}

Results in Table \ref{tab:sudoku_results} show that both Gödel-based approaches reach accuracies statistically equivalent to A-NeSI. It is worth noting that, to correctly classify a valid Sudoku instance, the model must simultaneously predict all 81 images correctly.
As a consequence, even a $0.994$ per-digit accuracy limits the expected grid accuracy to $0.994^{81} \approx 0.61$. 

Notably, the deterministic \godel Logic is the slowest method. In contrast, \gt is more than twice as fast, demonstrating a significant computational advantage. This efficiency stems from the mutual exclusivity constraints: while \gt uses the extremely efficient $shift$ function to enforce this constraint, in \godel logic we can not utilize such strategy, requiring more expensive operations to maintain numerical stability while enforcing the constraint.

Overall, these results validate our theoretical framework and its applicability in combination with neural networks, demonstrating that \godel logic and its variants are viable tools for neurosymbolic domains.

\section{Conclusions and Future Work}
\label{sec:conclusion}

In this work, we challenged the view of \godel logic as a mere continuous relaxation by proving its homomorphism to classical logic semantics. Moreover, thanks to its gradient sparsity, we proved that \godel optimization formally instantiates a discrete local search for satisfiability.

To address the local optima inherent in this deterministic search, we introduced the \godel Trick (GT), a stochastic reparameterization designed for better exploration. Beyond its empirical success, GT establishes a formal theoretical bridge between fuzzy optimization and probabilistic inference, acting as a Monte Carlo estimator for WMC.

Future work will focus on extending \gt with optimizations inspired by the local search algorithms literature, for instance with Tabu Search, to further improve convergence and robustness.
Additionally, we plan to evaluate \gt on more complex NeSy tasks and experiment with alternative noise distributions beyond Logistic and Uniform to understand their effect on performance and generalization.

Another promising direction involves integrating GT with generative models, exploiting the connections with probabilistic inference (see Theorem~\ref{th:prob_inf}). By interpreting unperturbed truth values as negative energies in a Gibbs distribution, GT could be used to combine multiple Energy-Based Models (EBMs), providing a differentiable mechanism to enforce logical constraint satisfaction on generated samples.

Finally, we aim to investigate the integration of the \godel Trick into models capable of learning knowledge, such as DSL~\cite{daniele2023deep}, to assess its potential in end-to-end neurosymbolic learning frameworks.

\subsection{Limitations}
While gradient sparsity is the very reason \bpg mimics LSA, it may also reduce convergence speed compared to dense-gradient logic. Furthermore, \gt remains susceptible to cycles and local optima, lacking the global deductive capabilities of modern SAT solvers. Finally, \gt inherits common NeSy challenges, such as the independence assumption between propositions~\cite{vanindependence}, which can cause the optimization to collapse into a single deterministic solution in discriminative learning settings, thereby losing the ability to model uncertainty. Additionally, \gt remains susceptible to reasoning shortcuts \cite{marconato2024not,marconatobears}.

\section*{Acknowledgments} We would like to express our gratitude to Samy Badreddine for the insightful discussions and valuable feedback, which had a significant impact on this work. We also thank Samuel Cognolato and Davide Bizzaro for their helpful comments and suggestions.

\section*{AI Declaration} LLMs have been used during the preparation of this manuscript exclusively as an editing tool. Additionally, they have been used for creating the python code that generated the plots of Figure~\ref{fig:virtual_imp}(c), Figure~\ref{fig:virtual_imp}(d), and Figure~\ref{fig:GT}. Finally, they have been used for the python script that converts the results of SAT experiments into the latex code for Table~\ref{tab:methods_comparison}, and for the longer version of the table in the Supplementary Materials. All generated content have been checked by the authors.

\bibliographystyle{kr}
\bibliography{references.bib}

@article{selsam2018learning,
  title={Learning a SAT solver from single-bit supervision},
  author={Selsam, Daniel and Lamm, Matthew and B{\"u}nz, Benedikt and Liang, Percy and de Moura, Leonardo and Dill, David L},
  journal={arXiv preprint arXiv:1802.03685},
  year={2018}
}

@article{li2022nsnet,
  title={Nsnet: A general neural probabilistic framework for satisfiability problems},
  author={Li, Zhaoyu and Si, Xujie},
  journal={Advances in Neural Information Processing Systems},
  volume={35},
  pages={25573--25585},
  year={2022}
}

@inproceedings{wang2019satnet,
  title={SATNet: Bridging deep learning and logical reasoning using a differentiable satisfiability solver},
  author={Wang, Po-Wei and Donti, Priya L and Wilder, Bryan and Kolter, J Zico},
  booktitle={International conference on machine learning},
  pages={6545--6554},
  year={2019},
  organization={PMLR}
}

@article{gupta1991theory,
  title={Theory of T-norms and fuzzy inference methods},
  author={Gupta, Madan M and Qi, J11043360726},
  journal={Fuzzy sets and systems},
  volume={40},
  number={3},
  pages={431--450},
  year={1991},
  publisher={Elsevier}
}

@inproceedings{daniele2023deep,
  title={Deep symbolic learning: discovering symbols and rules from perceptions},
  author={Daniele, Alessandro and Campari, Tommaso and Malhotra, Sagar and Serafini, Luciano},
  booktitle={Proceedings of the Thirty-Second International Joint Conference on Artificial Intelligence},
  pages={3597--3605},
  year={2023}
}

@article{rumelhart1986learning,
  title={Learning representations by back-propagating errors},
  author={Rumelhart, David E and Hinton, Geoffrey E and Williams, Ronald J},
  journal={nature},
  volume={323},
  number={6088},
  pages={533--536},
  year={1986},
  publisher={Nature Publishing Group UK London}
}

@article{gumbel1954statistical,
  title={Statistical theory of extreme valuse and some practical applications},
  author={Gumbel, Emil Julius},
  journal={Nat. Bur. Standards Appl. Math. Ser. 33},
  year={1954}
}

@article{pryor2022neupsl,
  title={Neupsl: Neural probabilistic soft logic},
  author={Pryor, Connor and Dickens, Charles and Augustine, Eriq and Albalak, Alon and Wang, William and Getoor, Lise},
  journal={arXiv preprint arXiv:2205.14268},
  year={2022}
}

@article{van2022analyzing,
  title={Analyzing differentiable fuzzy logic operators},
  author={van Krieken, Emile and Acar, Erman and van Harmelen, Frank},
  journal={Artificial Intelligence},
  volume={302},
  pages={103602},
  year={2022},
  publisher={Elsevier}
}

@article{badreddine2022logic,
  title={Logic tensor networks},
  author={Badreddine, Samy and Garcez, Artur d'Avila and Serafini, Luciano and Spranger, Michael},
  journal={Artificial Intelligence},
  volume={303},
  pages={103649},
  year={2022},
  publisher={Elsevier}
}

@article{diligenti2017semantic,
  title={Semantic-based regularization for learning and inference},
  author={Diligenti, Michelangelo and Gori, Marco and Sacca, Claudio},
  journal={Artificial Intelligence},
  volume={244},
  pages={143--165},
  year={2017},
  publisher={Elsevier}
}

@article{MARRA2024104062,
  title = {From Statistical Relational to Neurosymbolic Artificial Intelligence: {{A}} Survey},
  author = {Marra, Giuseppe and Duman{\v c}i{\'c}, Sebastijan and Manhaeve, Robin and De Raedt, Luc},
  year = {2024},
  journal = {Artificial Intelligence},
  volume = {328},
  pages = {104062},
  issn = {0004-3702},
  doi = {10.1016/j.artint.2023.104062}
}

@misc{feldstein2024mappingneurosymbolicailandscape,
      title={Mapping the Neuro-Symbolic AI Landscape by Architectures: A Handbook on Augmenting Deep Learning Through Symbolic Reasoning}, 
      author={Jonathan Feldstein and Paulius Dilkas and Vaishak Belle and Efthymia Tsamoura},
      year={2024},
      eprint={2410.22077},
      archivePrefix={arXiv},
      primaryClass={cs.AI},
      url={https://arxiv.org/abs/2410.22077}, 
}

@article{grespan2021evaluating,
  title={Evaluating relaxations of logic for neural networks: A comprehensive study},
  author={Grespan, Mattia Medina and Gupta, Ashim and Srikumar, Vivek},
  journal={arXiv preprint arXiv:2107.13646},
  year={2021}
}

@article{flinkow2024comparing,
  title={Comparing differentiable logics for learning with logical constraints},
  author={Flinkow, Thomas and Pearlmutter, Barak A and Monahan, Rosemary},
  journal={arXiv preprint arXiv:2407.03847},
  year={2024}
}

@inproceedings{slusarz2023logic,
  title={Logic of Differentiable Logics: Towards a Uniform Semantics of DL},
  author={Slusarz, Natalia and Komendantskaya, Ekaterina and Daggitt, Matthew L and Stewart, Robert and Stark, Kathrin},
  booktitle={Proceedings of 24th International Conference on Logic},
  volume={94},
  pages={473--493},
  year={2023}
}

@article{giunchiglia2024ccn+,
  title={CCN+: A neuro-symbolic framework for deep learning with requirements},
  author={Giunchiglia, Eleonora and Tatomir, Alex and Stoian, Mihaela C{\u{a}}t{\u{a}}lina and Lukasiewicz, Thomas},
  journal={International Journal of Approximate Reasoning},
  pages={109124},
  year={2024},
  publisher={Elsevier}
}

@inproceedings{daniele2019knowledge,
  title={Knowledge enhanced neural networks},
  author={Daniele, Alessandro and Serafini, Luciano},
  booktitle={PRICAI 2019: Trends in Artificial Intelligence: 16th Pacific Rim International Conference on Artificial Intelligence, Cuvu, Yanuca Island, Fiji, August 26--30, 2019, Proceedings, Part I 16},
  pages={542--554},
  year={2019},
  organization={Springer}
}

@article{daniele2023refining,
  title={Refining neural network predictions using background knowledge},
  author={Daniele, Alessandro and van Krieken, Emile and Serafini, Luciano and van Harmelen, Frank},
  journal={Machine Learning},
  volume={112},
  number={9},
  pages={3293--3331},
  year={2023},
  publisher={Springer}
}

@inproceedings{xu2018semantic,
  title={A semantic loss function for deep learning with symbolic knowledge},
  author={Xu, Jingyi and Zhang, Zilu and Friedman, Tal and Liang, Yitao and Broeck, Guy},
  booktitle={International conference on machine learning},
  pages={5502--5511},
  year={2018},
  organization={PMLR}
}

@article{manhaeve2018deepproblog,
  title={Deepproblog: Neural probabilistic logic programming},
  author={Manhaeve, Robin and Dumancic, Sebastijan and Kimmig, Angelika and Demeester, Thomas and De Raedt, Luc},
  journal={Advances in neural information processing systems},
  volume={31},
  year={2018}
}

@article{ahmed2022semantic,
  title={Semantic probabilistic layers for neuro-symbolic learning},
  author={Ahmed, Kareem and Teso, Stefano and Chang, Kai-Wei and Van den Broeck, Guy and Vergari, Antonio},
  journal={Advances in Neural Information Processing Systems},
  volume={35},
  pages={29944--29959},
  year={2022}
}

@article{chavira2008probabilistic,
  title={On probabilistic inference by weighted model counting},
  author={Chavira, Mark and Darwiche, Adnan},
  journal={Artificial Intelligence},
  volume={172},
  number={6-7},
  pages={772--799},
  year={2008},
  publisher={Elsevier}
}

@article{choi2020probabilistic,
  title={Probabilistic circuits: A unifying framework for tractable probabilistic models},
  author={Choi, Y and Vergari, Antonio and Van den Broeck, Guy},
  journal={UCLA. URL: http://starai. cs. ucla. edu/papers/ProbCirc20. pdf},
  pages={6},
  year={2020}
}

@inproceedings{kisa2014probabilistic,
  title={Probabilistic sentential decision diagrams},
  author={Kisa, Doga and Van den Broeck, Guy and Choi, Arthur and Darwiche, Adnan},
  booktitle={Fourteenth International Conference on the Principles of Knowledge Representation and Reasoning},
  year={2014}
}

@article{van2023nesi,
  title={A-nesi: A scalable approximate method for probabilistic neurosymbolic inference},
  author={van Krieken, Emile and Thanapalasingam, Thiviyan and Tomczak, Jakub and Van Harmelen, Frank and Ten Teije, Annette},
  journal={Advances in Neural Information Processing Systems},
  volume={36},
  pages={24586--24609},
  year={2023}
}

@inproceedings{NEURIPS2023_61202bb3,
  title = {Differentiable Sampling of Categorical Distributions Using the {{CatLog-derivative}} Trick},
  booktitle = {Advances in Neural Information Processing Systems},
  author = {De Smet, Lennert and Sansone, Emanuele and Zuidberg Dos Martires, Pedro},
  editor = {Oh, A. and Naumann, T. and Globerson, A. and Saenko, K. and Hardt, M. and Levine, S.},
  year = {2023},
  volume = {36},
  pages = {30416--30428},
  publisher = {Curran Associates, Inc.}
}

@article{marconato2024not,
  title={Not all neuro-symbolic concepts are created equal: Analysis and mitigation of reasoning shortcuts},
  author={Marconato, Emanuele and Teso, Stefano and Vergari, Antonio and Passerini, Andrea},
  journal={Advances in Neural Information Processing Systems},
  volume={36},
  year={2024}
}

@inproceedings{marconatobears,
  title={BEARS Make Neuro-Symbolic Models Aware of their Reasoning Shortcuts},
  author={Marconato, Emanuele and Bortolotti, Samuele and van Krieken, Emile and Vergari, Antonio and Passerini, Andrea and Teso, Stefano},
  booktitle={The 40th Conference on Uncertainty in Artificial Intelligence},
  year={2024}
}

@inproceedings{vanindependence,
  title={On the Independence Assumption in Neurosymbolic Learning},
  author={van Krieken, Emile and Minervini, Pasquale and Ponti, Edoardo and Vergari, Antonio},
  booktitle={Forty-first International Conference on Machine Learning},
  year={2024}
}

@article{valiant1979complexity,
  title={The complexity of enumeration and reliability problems},
  author={Valiant, Leslie G},
  journal={siam Journal on Computing},
  volume={8},
  number={3},
  pages={410--421},
  year={1979},
  publisher={SIAM}
}

@article{mohamed2020monte,
  title={Monte carlo gradient estimation in machine learning},
  author={Mohamed, Shakir and Rosca, Mihaela and Figurnov, Michael and Mnih, Andriy},
  journal={Journal of Machine Learning Research},
  volume={21},
  number={132},
  pages={1--62},
  year={2020}
}

@inproceedings{DBLP:journals/corr/KingmaW13,
  author       = {Diederik P. Kingma and
                  Max Welling},
  editor       = {Yoshua Bengio and
                  Yann LeCun},
  title        = {Auto-Encoding Variational Bayes},
  booktitle    = {2nd International Conference on Learning Representations, {ICLR} 2014,
                  Banff, AB, Canada, April 14-16, 2014, Conference Track Proceedings},
  year         = {2014},
  url          = {http://arxiv.org/abs/1312.6114},
  bibsource    = {dblp computer science bibliography, https://dblp.org}
}

@inproceedings{jang2022categorical,
  title={Categorical Reparameterization with Gumbel-Softmax},
  author={Jang, Eric and Gu, Shixiang and Poole, Ben},
  booktitle={International Conference on Learning Representations},
  year={2022}
}

@inproceedings{augustine2022visual,
  title={Visual sudoku puzzle classification: A suite of collective neuro-symbolic tasks},
  author={Augustine, Eriq and Pryor, Connor and Dickens, Charles and Pujara, Jay and Wang, William and Getoor, Lise},
  booktitle={International Workshop on Neural-Symbolic Learning and Reasoning (NeSy)},
  year={2022}
}

@article{hoos2000satlib,
  title={SATLIB: An online resource for research on SAT},
  author={Hoos, Holger H and St{\"u}tzle, Thomas},
  journal={Sat},
  volume={2000},
  pages={283--292},
  year={2000}
}

@article{moisil1935recherches,
  title={Recherches sur l’algebre de la logique},
  author={Moisil, Grigor Constantin},
  journal={Ann. Sci. Univ. Jassy},
  volume={22},
  number={3},
  pages={1--117},
  year={1935}
}

@article{kalman1958lattices,
  title={Lattices with involution},
  author={Kalman, John Arnold},
  journal={Transactions of the American Mathematical Society},
  volume={87},
  number={2},
  pages={485--491},
  year={1958}
}

@article{andreoni2025t,
  title={T-ILR: a neurosymbolic integration for LTLf},
  author={Andreoni, Riccardo and Buliga, Andrei and Daniele, Alessandro and Ghidini, Chiara and Montali, Marco and Ronzani, Massimiliano},
  journal={arXiv preprint arXiv:2508.15943},
  year={2025}
}

\end{document}